\documentclass[11pt]{article}

\usepackage[a4paper,margin=1in]{geometry}
\usepackage{amsmath,amssymb,amsthm,mathtools}
\usepackage{booktabs}
\usepackage{array}
\usepackage{tikz}
\usetikzlibrary{arrows.meta,positioning,fit,decorations.pathreplacing}
\usepackage{microtype}
\usepackage{enumitem}
\usepackage{xcolor}
\usepackage{hyperref}
\usepackage[numbers,sort&compress]{natbib}

\hypersetup{
  colorlinks=true,
  linkcolor=blue!60!black,
  citecolor=blue!60!black,
  urlcolor=blue!60!black,
  pdftitle={RIG-RoPE: Relation-Stratified Multimodal Attention with Instance-Local Rotary Geometry and Representation-Aware Traversal Coordinates},
  pdfauthor={Donggen Li}
}

\title{\textbf{RIG-RoPE: Relation-Stratified Multimodal Attention with Instance-Local Rotary Geometry and Representation-Aware Traversal Coordinates}}

\author{
  Donggen Li\thanks{This technical report establishes the problem formulation and method and includes controlled real-checkpoint Qwen2-VL activation/mechanism evidence plus tiny-model validation. It does not include a task-level Qwen benchmark, Qwen training, or evidence of large-scale empirical superiority.}\\
  Sichuan University\\
  \texttt{lidonggen0425@163.com}
}

\date{August 2026}

\newtheorem{definition}{Definition}
\newtheorem{assumption}{Assumption}
\newtheorem{proposition}{Proposition}
\newtheorem{theorem}{Theorem}
\newtheorem{corollary}{Corollary}

\newcommand{\Text}{\mathrm{Text}}
\newcommand{\Vision}{\mathrm{Vision}}
\newcommand{\Image}{\mathrm{Image}}
\newcommand{\Video}{\mathrm{Video}}
\newcommand{\R}{\mathbb{R}}

\newcommand{\inner}[2]{\left\langle #1,#2\right\rangle}
\newcommand{\q}{\mathbf{q}}
\newcommand{\kvec}{\mathbf{k}}
\newcommand{\x}{\mathbf{x}}
\newcommand{\calV}{\mathcal{V}}

\newcommand{\calJ}{\mathcal{J}}

\newcommand{\softmax}{\operatorname{softmax}}

\newcommand{\Shared}{\operatorname{SharedChart}}

\begin{document}
\maketitle

\begin{abstract}
Multimodal rotary positional encoding commonly applies temporal, height, and width phases to interleaved text, image, and video tokens. We identify two structural ambiguities. First, a cross-instance spatial displacement may be numerically available under preprocessing conventions while lacking an intrinsic geometric meaning unless a shared chart or registration is explicitly declared. Second, the scalar advance across visual blocks is often inherited from local coordinate extrema rather than defined as a representation-level traversal measure.

We propose \emph{RIG-RoPE}, which combines instance-local rotary geometry, attention stratified by relation, and representation-aware traversal coordinates. A gauge argument shows why raw cross-instance coordinate subtraction depends on independent chart choices. RIG-RoPE separates relation-homogeneous softmaxes, allocates their total mass using a common H/W-neutral LogSumExp statistic, and defines traversal extent to be additive over ordered slices and sublinear over parallel spatial scale. Text retains unit increments, image patches are simultaneous, and video accumulates over tokenizer-produced temporal tokens, including when $F=1$ remains a video regime. In a matched, inference-only Qwen2-VL-2B checkpoint experiment, text-only native and RIG paths were exactly equal, and the RIG path was exactly invariant to both a whole-chart single-image translation and a second-instance-only translation for two unregistered images. The native two-image path remained sensitive, while an H/W-collapse control confirmed that RIG retained same-instance spatial effects. Visual embeddings and all 729 model parameters were exactly preserved. Across three seeds of a frozen tiny task, the RIG construction likewise reproduced zero clean-to-Gauge logit change, whereas the raw-H/W baseline changed for every seed. The paired Gauge-accuracy differences were $+2/72$, $0$, and $0$, so the preregistered stability gate failed and stable task utility remains unestablished. These results support the specified activation and Gauge mechanisms, not stable task improvement or empirical superiority.
\end{abstract}

\section{Introduction}

Multimodal large language models (MLLMs) increasingly process text, images, and videos inside a unified transformer sequence. A common design serializes visual inputs into patch tokens~\citep{dosovitskiy2020vit} and interleaves them with text tokens, allowing self-attention to mediate both semantic and spatial interaction. Position encoding is therefore not merely an implementation detail: it defines which relative relations are visible to attention.

RoPE~\citep{su2021roformer} encodes one-dimensional relative position through rotations of query and key vectors. Recent MLLMs extend this idea to visual data by decomposing each attention head into temporal, height, and width subspaces. Qwen2-VL, for example, introduces M-RoPE to support text, image, and video tokens under dynamic visual resolution~\citep{wang2024qwen2vl}. Related work studies video-oriented rotary layouts~\citep{liu2025vrope,videorope2025} and cross-modal positional bias~\citep{wang2025circlerope}.

This report argues that the central issue is not only how to choose better static coordinates, but when a coordinate difference should be treated as geometrically meaningful. A visual patch has a height/width coordinate relative to a chart induced by its preprocessing pipeline. A text token does not inhabit that chart. Patches from different images may share a conventional normalized screen coordinate, yet that coordinate is generally not invariant to independent cropping, padding, resizing, or origin choices. A cross-instance displacement becomes geometrically usable only when the model or task explicitly declares a shared chart, registration, or canonical-coordinate contract. RIG-RoPE adopts instance-local geometry as the conservative default while allowing registered cross-instance pairs.

A second issue appears in the temporal subspace. The scalar coordinate used across interleaved blocks should express how far the model traverses through its represented context, not the raster length of a visual block and not necessarily physical or cognitive time. Naive placeholder schemes may assign one step to either a text token or a complete image. M-RoPE is more structured: image tokens share a temporal ID, visual H/W IDs vary over the grid, and the following modality starts after the maximum position ID of the preceding modality~\citep{wang2024qwen2vl}. Consequently, a visual block already induces a resolution-dependent global advance. However, that advance is produced indirectly by the maximum of local coordinate axes. It does not state which axes are ordered, which tokens are parallel, or how video duration should behave under temporal partitioning. We instead define an explicit \emph{multimodal traversal coordinate}: ordered temporal slices accumulate additively, whereas tokens within one parallel spatial slice contribute only a sublinear spatial scale.

We introduce \emph{RIG-RoPE}, named for its \emph{relation-stratified, instance-local geometry}. The framework combines this geometry with representation-aware traversal coordinates and follows two simple rules:
\begin{quote}
\emph{Apply height/width rotary encoding only when a shared visual chart is declared. Do not let spatially registered and unregistered relations compete through uncalibrated H/W phases. Measure global context separation on a traversal axis that is additive along ordered slices and sublinear along parallel spatial slices.}
\end{quote}
For each query, RIG-RoPE forms relation-homogeneous key groups. Same-instance visual interactions retain ordinary H/W RoPE; text--text interactions retain full one-dimensional RoPE; and unregistered interactions use traversal-aware temporal rotation without an asserted H/W displacement. Each group is normalized internally. A common H/W-neutral evidence statistic then determines how much total attention mass each active relation group receives. The method does not block attention: every causally visible key remains reachable through exactly one group.

\paragraph{Intuitive overview.}
RIG-RoPE changes two questions asked by positional attention. First, it asks whether a spatial displacement is declared meaningful for the query--key pair. Second, it asks whether scores produced by different positional rules are calibrated well enough to share a denominator. A missing spatial relation is neither a special distance nor a noisy zero displacement. Setting its RoPE rotation to $R(0)=I$ gives a self-aligned query--key pair the largest possible phase contribution, although this ordering need not hold for arbitrary content vectors. The method therefore uses $I$ only as an explicitly typed ``no H/W transform'' inside the unregistered branch and prevents relation-specific H/W phases from controlling cross-group mass. Independently, it uses a content-independent traversal coordinate determined by the modality regime and post-tokenization structure rather than estimated scene complexity or semantic redundancy.

\paragraph{Contributions.}
This technical report makes four contributions:
\begin{enumerate}[leftmargin=1.5em]
  \item \textbf{Instance-local geometry and null-relation analysis.} We formalize spatial comparability through an explicit shared-chart contract, prove the gauge dependence of raw unregistered displacement, and characterize identity fallback through a deterministic self-aligned result and an assumption-explicit expectation result.
  \item \textbf{Relation-stratified attention with common evidence.} We route textual, spatially registered, and unregistered pairs to relation-homogeneous normalizers. A common H/W-neutral gate controls group mass, admits an exact common-score consistency theorem, and distinguishes the default global-mass semantics from an optional group-balanced prior.
  \item \textbf{Ordered/parallel traversal coordinates.} We derive a scalar context axis that is additive along ordered slices and sublinear along parallel spatial scale, preserve text offsets and image simultaneity, and obtain video partition additivity with $F$ defined as the tokenizer's temporal-token count.
  \item \textbf{Controlled validation separating mechanism invariance from task utility.} In a matched full-checkpoint Qwen2-VL-2B inference run, we verify exact text-only native/RIG equivalence, exact RIG invariance to two declared Gauge translations, native Gauge sensitivity, and retained same-instance spatial effects. A frozen three-seed tiny comparison reproduces the Gauge mechanism distinction but does not establish stable accuracy utility.
\end{enumerate}

\section{Background: Multidimensional RoPE in MLLMs}

\subsection{One-Dimensional RoPE}

Let $\q_i,\kvec_j\in\R^d$ denote the query and key vectors at positions $i$ and $j$. RoPE applies a block-diagonal rotation matrix $R(p)$ to each vector:
\begin{equation}
  s_{ij}^{\mathrm{1D}}
  =
  \inner{R(s_i)\q_i}{R(s_j)\kvec_j}
  =
  \inner{\q_i}{R(s_j-s_i)\kvec_j},
  \label{eq:rope-relative}
\end{equation}
where $s_i$ is the one-dimensional sequence index. The key property is relative-position dependence: the attention logit depends on the offset $s_j-s_i$ rather than only on absolute positions.

\subsection{Multidimensional RoPE}

For visual tokens, M-RoPE assigns two-dimensional RoPE frequency blocks to temporal, height, and width axes. Let $\Omega_t,\Omega_h,\Omega_w$ be disjoint sets of frequency blocks and let $P_a$ select the blocks assigned to axis $a$. Then
\begin{equation}
  \Omega_t\mathbin{\dot\cup}\Omega_h\mathbin{\dot\cup}\Omega_w=\Omega,
  \qquad
  \q_i^a=P_a\q_i,
  \quad
  \kvec_j^a=P_a\kvec_j.
\end{equation}
The sets may be contiguous channel ranges, round-robin interleaved blocks, or all blocks in heads dedicated to one axis. This notation therefore covers classical M-RoPE, MRoPE-Interleave, and head-wise MHRoPE~\citep{huang2025revisiting}.
Given token coordinates $(t_i,h_i,w_i)$, the attention logit is commonly written as
\begin{equation}
  s_{ij}^{M}
  =
  \inner{\q_i^t}{R_t(\Delta t_{ij})\kvec_j^t}
  +
  \inner{\q_i^h}{R_h(\Delta h_{ij})\kvec_j^h}
  +
  \inner{\q_i^w}{R_w(\Delta w_{ij})\kvec_j^w},
  \label{eq:mrope}
\end{equation}
where $\Delta t_{ij}=t_j-t_i$, $\Delta h_{ij}=h_j-h_i$, and $\Delta w_{ij}=w_j-w_i$ up to a sign convention. Equation~\eqref{eq:mrope} is valid when each displacement is well defined. The problem studied here is that the formula is often applied beyond that domain.

\section{Failure Modes of Static Multimodal Position Assignment}

\subsection{Ill-Posed Spatial Relations}

Static multimodal position assignment applies the same coordinate arithmetic to every query-key pair after flattening the input sequence. This is problematic for spatial RoPE because not every token pair admits a shared image-plane coordinate system. A text token has no height/width location in an image, and two patches from different visual instances have coordinates in different local charts. Thus static height/width IDs can create artificial spatial proximity between a text token and a particular image region, or between unrelated patches from different images.

\subsection{Implicit and Axis-Coupled Global Advance}

In a flattened multimodal sequence, a scalar ordering coordinate must coexist with modality-specific local coordinates. A naive placeholder counter can give both sequences
\begin{equation}
  [\Text_1,\Text_2,\Text_3],
  \qquad
  [\Text_1,\Image_2,\Text_3]
\end{equation}
the same temporal offset between $\Text_1$ and $\Text_3$ if the middle text token and the middle image each advance the counter by one unit. M-RoPE does not generally make this exact equal-step assignment: its next modality begins after the maximum ID of the previous modality, so a visual grid can advance later positions by a scale comparable to its largest axis~\citep{wang2024qwen2vl}. The remaining issue is more precise. A maximum-axis rule couples global context advance to a local coordinate convention, treats temporal and spatial extrema through the same aggregation, and does not guarantee additivity when a video is divided into consecutive temporal segments. The desired coordinate should instead expose an explicit scalar measure with three properties: preservation of ordinary text offsets, additivity along genuinely ordered axes, and sublinear contribution from tokens that occur in parallel inside a spatial slice.

\section{Unified Problem Formulation}

\subsection{Token Attributes}

We consider an interleaved sequence of text and visual tokens. Each token $i$ has:
\begin{equation}
  M_i \in \{\Text,\Vision\},
  \qquad
  I_i \in \mathbb{N}_0,
\end{equation}
where $M_i$ is the modality and $I_i$ is the instance identifier. Text tokens use $I_i=0$ by convention. Tokens from the same image or continuous video segment share the same positive instance identifier.

We also group the flattened sequence into semantic blocks. A block $B_b$ is either a single text token, one image instance, or one video instance. Let $b(i)$ denote the block containing token $i$. Each block has a modality label $m_b\in\{\Text,\Image,\Video\}$, a positive traversal extent $D_b$, and a cumulative start anchor
\begin{equation}
  a_b = \sum_{r<b} D_r.
  \label{eq:block-anchor}
\end{equation}
The representation-aware traversal coordinate of token $i$ is denoted by $\tau_i$. The detailed construction of $D_b$ and $\tau_i$ is given in Section~\ref{sec:duration}.

\begin{table}[t]
\centering
\caption{Main notation used in RIG-RoPE.}
\label{tab:notation}
\begin{tabular}{>{\raggedright\arraybackslash}p{0.18\linewidth}>{\raggedright\arraybackslash}p{0.72\linewidth}}
\toprule
Symbol & Meaning \\
\midrule
$M_i$ & modality of token $i$, either text or vision \\
$I_i$ & visual instance identifier; text tokens use $I_i=0$ \\
$B_b$ & semantic block: one text token, one image instance, or one video instance \\
$m_b$ & modality regime of block $B_b$: text, image, or video \\
$U_b$ & number of ordered slices in block $B_b$; for video, $U_b=F_b$ \\
$H_{b,u},W_{b,u}$ & post-tokenization spatial grid of ordered slice $u$ \\
$F_b$ & number of temporal tokens output by the visual tokenizer for video block $B_b$ \\
$\ell(H,W)$ & characteristic spatial scale, homogeneous under uniform resizing \\
$\beta_V$ & video parallel-scale exponent in the fixed family, $0\leq\beta_V\leq1$ \\
$D_I^\star,d_V^\star$ & reference image extent and reference video per-slice extent \\
$D_b$ & traversal extent assigned to block $B_b$ \\
$a_b$ & cumulative start anchor of block $B_b$ on the traversal axis \\
$\tau_i$ & representation-aware traversal coordinate of token $i$ \\
$\chi_m$ & local image-plane coordinate chart of visual instance $m$ \\
$\mathcal{G}$ & declared chart contract specifying which visual pairs share a coordinate system \\
$r_{ij}$ & relation class of pair $(i,j)$: textual $T$, spatially registered $S$, or unregistered $U$ \\
$\calJ_i^r$ & causally visible keys of relation class $r$ for query $i$ \\
$a_{ij}^{r}$ & within-group normalized attention weight \\
$\gamma_i^r$ & relation-level gate; equivalently, total attention mass assigned to group $r$ \\
$c_{ij}$ & common H/W-neutral evidence used only to compute $\gamma_i^r$ \\
\bottomrule
\end{tabular}
\end{table}

For a visual instance $m$, let $\calV^{(m)}$ denote its visual token grid and let
\begin{equation}
  \chi_m : \calV^{(m)} \rightarrow \R^2
\end{equation}
be a local image-plane coordinate chart for image patches or per-frame video patches. A token $i$ in instance $m$ has local screen-space coordinate
\begin{equation}
  \x_i^{(m)}
  =
  \chi_m(i)
  =
  (h_i^{(m)}, w_i^{(m)}) \in \R^2.
\end{equation}

\subsection{Spatial Well-Posedness}

The preprocessing pipeline always supplies numerical H/W indices, but numerical availability is weaker than intrinsic geometric comparability. We therefore make the modeling contract explicit. Let $\mathcal{G}$ declare, for selected visual pairs, maps $\psi_i$ and $\psi_j$ from their local charts into a common chart $\mathcal{C}_{ij}$. The contract may be induced by instance identity, external registration, or an intentional task-specific canonical screen coordinate.

\begin{definition}[Spatially comparable pair under $\mathcal{G}$]
A query--key pair $(i,j)$ is spatially comparable under the declared chart contract $\mathcal{G}$ if both tokens are visual and a common chart has been supplied:
\begin{equation}
  \Shared_{\mathcal{G}}(i,j)
  \iff
  M_i=M_j=\Vision
  \ \land\ 
  \exists\,(\mathcal{C}_{ij},\psi_i,\psi_j)\in\mathcal{G}.
  \label{eq:shared-chart}
\end{equation}
\end{definition}

For such pairs the declared spatial displacement is
\begin{equation}
  \Delta \x_{ij}^{\mathcal{G}}
  =
  \psi_j(\x_j)-\psi_i(\x_i)
  =
  (\Delta h_{ij}^{\mathcal{G}},\Delta w_{ij}^{\mathcal{G}})
  \label{eq:intra-delta}
\end{equation}
a valid object in the declared chart. The default contract is instance local: if $I_i=I_j$, both maps are the identity on that instance's chart. Externally registered images may also enter the spatial branch. Conversely, normalized screen coordinates across unrelated images are treated as a convention rather than an intrinsic relation unless the task explicitly opts into that convention. Text--vision pairs remain non-spatial because text has no image-plane coordinate. For video, a shared chart should be interpreted as screen-space geometry rather than a strict physical-world manifold; camera or object motion can make physical correspondence approximate.

\section{Spatial Well-Posedness and Gauge Invariance}

This section gives the main theoretical reason for instance-local gating as a default. It is mathematically possible to subtract normalized screen coordinates from two images, but the result is convention dependent: independent chart reparameterizations change it unless a coupling between charts is declared. The theorem below concerns this \emph{intrinsicness}, not the mere existence of two numbers that can be subtracted.

\subsection{Gauge Freedom of Visual Instances}

\begin{assumption}[Independent chart gauges]
Each visual instance $m$ admits an independent change of coordinate origin
\begin{equation}
  \x^{(m)} \mapsto \x^{(m)}+b_m,
  \qquad b_m\in\R^2,
  \label{eq:gauge}
\end{equation}
that changes the chart representation but not the semantic content of the instance.
\end{assumption}

Translation freedom alone is sufficient for the argument below. Cropping, padding, resizing, and per-image normalization introduce additional chart changes, but claiming invariance to all such transformations would be stronger than standard axial RoPE itself supports. Unless the task provides an explicit alignment between images, each instance may at least choose its origin independently.

\begin{proposition}[Intra-instance displacement is translation-gauge invariant]
If $i$ and $j$ belong to the same visual instance $m$, then under any shared translation $\x^{(m)}\mapsto \x^{(m)}+b_m$,
\begin{equation}
  \Delta \x_{ij}^{(m)}
  =
  \x_j^{(m)}-\x_i^{(m)}
\end{equation}
is invariant.
\end{proposition}

\begin{proof}
After the transformation,
\begin{equation}
  (\x_j^{(m)}+b_m)-(\x_i^{(m)}+b_m)
  =
  \x_j^{(m)}-\x_i^{(m)}.
\end{equation}
Thus a same-instance relative displacement does not depend on the arbitrary coordinate origin.
\end{proof}

\begin{theorem}[Raw cross-instance displacement is not intrinsic]
Let $i$ belong to visual instance $m$ and $j$ belong to visual instance $n\neq m$. Suppose no registration map $\phi_{m\rightarrow n}:\calV^{(m)}\rightarrow\calV^{(n)}$ is given. Then
\begin{equation}
  \x_j^{(n)}-\x_i^{(m)}
  \label{eq:cross-delta}
\end{equation}
is not a gauge-invariant geometric object.
\end{theorem}

\begin{proof}
Under independent translations of the two coordinate charts,
\begin{equation}
  \x_i^{(m)}\mapsto \x_i^{(m)}+b_m,
  \qquad
  \x_j^{(n)}\mapsto \x_j^{(n)}+b_n.
\end{equation}
The apparent cross-instance displacement becomes
\begin{equation}
  (\x_j^{(n)}+b_n)-(\x_i^{(m)}+b_m)
  =
  (\x_j^{(n)}-\x_i^{(m)})+(b_n-b_m).
\end{equation}
Because $b_m$ and $b_n$ are independent, $b_n-b_m$ can be arbitrary. Therefore the value of Eq.~\eqref{eq:cross-delta} is determined by chart choices rather than by instance-independent geometry. A preprocessing convention can still assign it a numerical value, but that value is not intrinsic unless $\mathcal{G}$ couples the two charts.
\end{proof}

\begin{corollary}[M-RoPE cross-instance phases are gauge-dependent]
For visual instances $m\neq n$, any height/width RoPE term of the form
\begin{equation}
  \inner{\q_i^h}{R_h(h_j^{(n)}-h_i^{(m)})\kvec_j^h}
  +
  \inner{\q_i^w}{R_w(w_j^{(n)}-w_i^{(m)})\kvec_j^w}
\end{equation}
depends on arbitrary chart choices unless a cross-instance registration is supplied.
\end{corollary}

This is the reason RIG-RoPE avoids ordinary height/width rotation for text--vision pairs and, by default, for visual pairs across different instances. Registered cross-instance pairs are an explicit exception. An unregistered relation should not be confused with zero displacement: identity rotation removes the H/W phase, but inside a shared flat softmax it can also change score statistics relative to genuinely rotated pairs.

\subsection{Null--Zero Conflation}

Removing an unregistered displacement does not yet specify how its logit should compete with a valid spatial logit. The most direct fallback sets the unsupported displacement to zero, hence $R(0)=I$. This is not generally calibrated inside a shared softmax.

\begin{theorem}[Self-aligned null-relation privilege]
Consider one two-dimensional RoPE frequency block with frequency $\omega$ and identical content vectors $u$ in the query and key. Let relation $A$ be a spatially registered pair with displacement $\Delta$, and let relation $B$ be unregistered but represented by identity fallback. Then
\begin{align}
  s_A &= u^\top R(\omega\Delta)u
      = \lVert u\rVert^2\cos(\omega\Delta),\\
  s_B &= u^\top Iu=\lVert u\rVert^2,
\end{align}
and therefore $s_B\geq s_A$. The inequality is strict whenever $\omega\Delta\notin2\pi\mathbb{Z}$ and $u\neq0$.
\end{theorem}

\begin{proof}
For a planar rotation, the skew-symmetric sine component vanishes in the quadratic form $u^\top R(\theta)u$, leaving $\lVert u\rVert^2\cos\theta$. Since $\cos\theta\leq1$, the claim follows. For a multidimensional RoPE subspace, the gap is the sum
\begin{equation}
  s_B-s_A
  =
  \sum_k \lVert u_k\rVert^2
  \left[1-\cos(\omega_k\Delta)\right]
  \geq 0.
  \label{eq:null-privilege-gap}
\end{equation}
\end{proof}

The theorem is a deterministic existence result for the self-aligned limiting case. It does \emph{not} state that identity rotation maximizes $q^\top R(\theta)k$ for arbitrary $q$ and $k$. In general,
\begin{equation}
  q^\top R(\theta)k
  =
  (q^\top k)\cos\theta+(q^\top Jk)\sin\theta,
  \qquad
  J=\begin{bmatrix}0&-1\\1&0\end{bmatrix},
  \label{eq:general-rotary-bilinear}
\end{equation}
so a nonzero rotation may align two initially nonparallel vectors better than the identity. The result nevertheless exposes a structural failure mode: a missing relation is assigned the unique phase that maximizes self-similarity.

\begin{proposition}[Expected privilege under isotropic correlation]
For one two-dimensional frequency block, assume
\begin{equation}
  \mathbb{E}[kq^\top]=\beta I,
  \qquad \beta>0.
  \label{eq:isotropic-correlation}
\end{equation}
Then
\begin{equation}
  \mathbb{E}[q^\top R(\theta)k]
  =2\beta\cos\theta
  \leq 2\beta
  =\mathbb{E}[q^\top k].
  \label{eq:expected-null-privilege}
\end{equation}
\end{proposition}

\begin{proof}
Using $q^\top R(\theta)k=\operatorname{tr}(R(\theta)kq^\top)$ and Eq.~\eqref{eq:isotropic-correlation}, the expectation is $\beta\operatorname{tr}(R(\theta))=2\beta\cos\theta$.
\end{proof}

Equation~\eqref{eq:expected-null-privilege} establishes an average bias only under the stated distributional assumption. Together, the theorem and proposition motivate explicit calibration; they do not imply that every unregistered key outranks every registered key.

Two further observations are deferred to Appendix~\ref{app:null-results}. No fixed content-independent fallback reproduces every valid rotary displacement pointwise, and Gaussian displacement marginalization models noisy registration rather than the absence of a shared chart. These results narrow the role of fixed fallbacks without implying that relation-stratified normalization is the only possible solution.

\section{Why Static Coordinates and a Flat Softmax Are Insufficient}

\begin{theorem}[Static assignment incompatibility]
Assume one static H/W coordinate assignment and shared non-degenerate RoPE frequencies are used for all pairs, without relation-conditioned routing. The following two requirements cannot both hold for all content vectors:
\begin{enumerate}[leftmargin=1.5em]
  \item \textbf{Intra-instance spatial faithfulness:} distinct patch displacements in one image induce non-identical H/W rotations.
  \item \textbf{Text--vision spatial invariance:} for a fixed text query and equal visual content placed at any two patch coordinates, the H/W operator in the text--vision logit is identical.
\end{enumerate}
\end{theorem}

\begin{proof}
Consider the height axis. Text--vision invariance requires
\begin{equation}
  R_h(h_{p_1}-h_a)=R_h(h_{p_2}-h_a)
\end{equation}
for any two patch coordinates $h_{p_1},h_{p_2}$, because the statement must hold for every query and equal key content. Multiplying by the inverse text-side rotation gives $R_h(h_{p_1})=R_h(h_{p_2})$. Under non-degenerate frequencies this collapses the distinction between those patch coordinates, contradicting intra-instance spatial faithfulness. The width axis is analogous. Hence the operator must depend on relation type, not only on static token coordinates.
\end{proof}

Pairwise routing alone is still incomplete if the routed scores share one denominator. Identity-routed unregistered logits and rotated registered logits can have different phase statistics, as Eq.~\eqref{eq:null-privilege-gap} demonstrates in the self-aligned case. The remaining problem is therefore \emph{calibration across heterogeneous softmax groups}: relation-specific H/W scores should not determine one another's total probability mass without an explicit calibration rule.

\section{Representation-Aware Traversal Coordinates}
\label{sec:duration}

The spatial argument above asks whether a spatial displacement is well defined. The temporal argument asks a different question: what scalar coordinate should order tokens that belong to heterogeneous representation geometries? We introduce a \emph{multimodal traversal coordinate} $\tau$. It measures represented context extent on the RoPE axis. It is not defined as wall-clock time, human reading time, semantic information content, or flattened token count. Human inspection provides intuition, but the mathematical object is determined only by the modality regime and the post-tokenization structure.

This distinction is important for videos. Throughout this paper,
\begin{equation}
  F_b
  \equiv
  \text{the number of temporal tokens in video block $B_b$ after visual tokenization},
  \label{eq:F-definition}
\end{equation}
including any temporal patching or merging performed by the visual tokenizer. It is neither the number of decoded source frames nor a content-dependent ``effective frame'' estimate. Under a fixed tokenizer, the construction is therefore content independent and deterministic.

\subsection{Ordered Slices and Parallel Tokens}

The relevant distinction is not tensor rank but the role of each axis in traversal. We decompose every block $B_b$ into $U_b$ \emph{ordered slices}. Tokens within one slice coexist on parallel spatial axes. Text has one token per ordered slice; an image has one ordered slice containing its H/W grid; a video has $F_b$ ordered temporal slices, each containing one H/W grid.

We impose the following structural requirements.

\begin{assumption}[Traversal structure]
For a fixed tokenizer and modality regime:
\begin{enumerate}[leftmargin=1.5em]
  \item \textbf{Ordered-axis additivity.} Concatenating consecutive ordered slices adds their extents.
  \item \textbf{Parallel-axis sublinearity.} Increasing the number of simultaneous spatial tokens may increase the extent, but no faster than their characteristic linear spatial scale.
  \item \textbf{Content independence.} Two blocks with the same modality label and the same post-tokenization grid receive the same coordinates, regardless of semantic content.
  \item \textbf{Text compatibility.} Each text token contributes one unit, so a text-only sequence exactly recovers ordinary one-dimensional RoPE offsets.
  \item \textbf{Phase-range compatibility.} At a reference tokenizer grid, the traversal extent is anchored to the position range already seen by the base model rather than introducing an arbitrary new phase scale.
\end{enumerate}
\end{assumption}

Let $\ell(H,W)>0$ be a characteristic spatial scale satisfying first-order scale covariance,
\begin{equation}
  \ell(cH,cW)=c\,\ell(H,W),
  \qquad c>0.
  \label{eq:scale-covariance}
\end{equation}
Useful choices include
\begin{equation}
  \ell_{\mathrm{area}}(H,W)=\sqrt{HW},\qquad
  \ell_{\mathrm{diag}}(H,W)=\sqrt{H^2+W^2},\qquad
  \ell_{\max}(H,W)=\max(H,W).
  \label{eq:spatial-scales}
\end{equation}
The area scale is aspect-ratio agnostic and recovers the original proposal; the diagonal and maximum-axis scales remain sensitive to elongated grids. The theory requires scale covariance, not one unique choice.

Ordered-axis additivity implies
\begin{equation}
  D_b
  =
  \sum_{u=1}^{U_b} d_{b,u}.
  \label{eq:additive-duration}
\end{equation}
For modality regime $m$, choose a fixed reference grid with scale $\ell_m^\star$ and a reference per-slice extent $d_m^\star$. We use the scale-covariant family
\begin{equation}
  d_{b,u}
  =
  d_{m_b}^\star
  \left(
    \frac{\ell(H_{b,u},W_{b,u})}{\ell_{m_b}^\star}
  \right)^{\beta_{m_b}},
  \label{eq:slice-duration}
\end{equation}
where $0\leq\beta_m\leq1$ controls the sublinear contribution of parallel spatial scale. Combining Eqs.~\eqref{eq:additive-duration} and~\eqref{eq:slice-duration} gives
\begin{equation}
  \boxed{
  D_b
  =
  d_{m_b}^\star
  \sum_{u=1}^{U_b}
  \left(
    \frac{\ell(H_{b,u},W_{b,u})}{\ell_{m_b}^\star}
  \right)^{\beta_{m_b}}
  }.
  \label{eq:unified-duration}
\end{equation}
For text we use the convention $H=W=\ell_T^\star=d_T^\star=1$ and $\beta_T=0$. The power law is an explicit inductive bias, not a uniquely forced theorem. This parameterization removes the redundancy between an arbitrary normalization count and an arbitrary multiplier: the reference grid is fixed, while $d_m^\star$ has the direct meaning of extent at that grid.

\subsection{Text, Image, and Video Specializations}

For a text token, the convention above gives
\begin{equation}
  D_b^{T}=1.
  \label{eq:text-duration}
\end{equation}

For an image with an $H_b\times W_b$ post-tokenization grid, all patches belong to one parallel slice. We take $\beta_I=1$ so that image extent follows a characteristic linear scale:
\begin{equation}
  D_b^{I}
  =
  D_I^\star
  \left(
    \frac{\ell(H_b,W_b)}{\ell(H_I^\star,W_I^\star)}
  \right).
  \label{eq:image-duration-new}
\end{equation}
With $\ell=\sqrt{HW}$, Eq.~\eqref{eq:image-duration-new} becomes
\begin{equation}
  D_b^I=D_I^\star
  \sqrt{\frac{H_bW_b}{H_I^\star W_I^\star}},
  \label{eq:image-area-specialization}
\end{equation}
which is a linear scale under similarity resizing but not the geometric diameter of an arbitrarily elongated grid. Equation~\eqref{eq:spatial-scales} makes that modeling choice explicit.

For a video, the ordered slices are the temporal tokens output by the visual tokenizer. If temporal slice $f$ has spatial grid $H_{b,f}\times W_{b,f}$, then
\begin{equation}
  D_b^{V}
  =
  d_V^\star
  \sum_{f=1}^{F_b}
  \left(
    \frac{\ell(H_{b,f},W_{b,f})}{\ell(H_V^\star,W_V^\star)}
  \right)^{\beta_V},
  \qquad
  0\leq\beta_V\leq1.
  \label{eq:video-duration-general}
\end{equation}
For a fixed spatial grid this reduces to
\begin{equation}
  \boxed{
  D_b^{V}
  =
  d_V^\star F_b
  \left(
    \frac{\ell(H_b,W_b)}{\ell(H_V^\star,W_V^\star)}
  \right)^{\beta_V}
  }.
  \label{eq:video-duration-fixed}
\end{equation}
The linear factor $F_b$ follows from ordered-axis additivity. At $\beta_V=0$, video extent depends only on temporal-token count. At $\beta_V=1$, each temporal slice contributes an image-like linear spatial scale. Intermediate values express the content-independent prior that a continuous stream may need a weaker per-slice spatial correction than a standalone image. For $\ell=\sqrt{HW}$, the original exponent is recovered by $\gamma_V=\beta_V/2$.

Phase-range compatibility fixes the remaining scale without introducing two redundant constants. Choose $D_I^\star$ and $d_V^\star$ so the reference image and reference video reproduce, as closely as possible, the scalar position spans used by the base model:
\begin{equation}
  D_I^\star\approx D_{I,\mathrm{base}}^\star,
  \qquad
  F_V^\star d_V^\star\approx D_{V,\mathrm{base}}^\star.
  \label{eq:phase-range-calibration}
\end{equation}
The construction is therefore a one-parameter family in $\beta_V$ once the base model, tokenizer, reference grids, and spatial-scale function are fixed. Any fixed member introduces no additional learned parameters; learning the calibration is an optional extension.

\subsection{Why Isotropic Token-Volume Scaling Is Insufficient}

A tempting unified rule is $D=N^{1/d}$, yielding exponents $1$, $1/2$, and $1/3$ for text, images, and videos. This recovers a characteristic linear scale from an isotropic token volume, but video is not isotropic: its temporal axis is ordered, while H/W tokens within a slice are parallel. In particular, the rule
\begin{equation}
  D_{mathrm{iso}}=(FHW)^{1/3}
\end{equation}
assigns the same extent to videos with the same total token count even when their temporal token counts differ. It also fails temporal-partition additivity. Splitting a video with $N$ tokens into two equal temporal pieces changes the total from $N^{1/3}$ to
\begin{equation}
  2\left(\frac{N}{2}\right)^{1/3}
  =2^{2/3}N^{1/3}.
\end{equation}
Moreover, distributing $D_{mathrm{iso}}$ over $F$ temporal slices produces adjacent spacing proportional to $F^{-2/3}$ at fixed H/W, so the local coordinate scale changes with the total clip length. Equation~\eqref{eq:video-duration-fixed} removes all three artifacts by treating temporal and spatial axes anisotropically.

\subsection{Coordinates Within a Block}

Each block occupies $[a_b,a_b+D_b]$, where $a_b$ is defined in Eq.~\eqref{eq:block-anchor}. Text and image tokens are placed at the block center:
\begin{equation}
  \tau_i
  =
  a_{b(i)}+\frac{1}{2}D_{b(i)},
  \qquad
  B_{b(i)}\in\{\Text,\Image\}.
  \label{eq:text-image-tau-new}
\end{equation}
Thus all patches of one image share the same traversal coordinate. Their H/W relations are encoded only in the spatial RoPE subspaces; raster order does not create a fictitious temporal order.

For video slice $f$, define
\begin{equation}
  d_{b,f}^{V}
  =
  d_V^\star
  \left(
    \frac{\ell(H_{b,f},W_{b,f})}{\ell(H_V^\star,W_V^\star)}
  \right)^{\beta_V}.
  \label{eq:video-local-duration}
\end{equation}
Every spatial token $i$ whose temporal-token index is $f_i$ receives
\begin{equation}
  \boxed{
  \tau_i
  =
  a_{b(i)}
  +
  \sum_{r<f_i}d_{b(i),r}^{V}
  +
  \frac{1}{2}d_{b(i),f_i}^{V}
  }.
  \label{eq:video-tau-new}
\end{equation}
All spatial patches in the same temporal slice share $\tau_i$. For fixed resolution,
\begin{equation}
  \tau_i
  =
  a_{b(i)}
  +
  \left(f_i-\frac{1}{2}\right)
  d_V^\star
  \left(
    \frac{\ell(H_b,W_b)}{\ell(H_V^\star,W_V^\star)}
  \right)^{\beta_V}.
  \label{eq:video-tau-fixed}
\end{equation}

An optional conservative interpolation is performed at the slice-extent level, never between $\tau_i$ and the flattened raster index. Let $d_{b,u}^{\mathrm{base}}=1$ denote a unit extent for each ordered slice and let $d_{b,u}^{\mathrm{trav}}$ be Eq.~\eqref{eq:slice-duration}. Then
\begin{equation}
  d_{b,u}^{(\lambda)}
  =
  (1-\lambda)d_{b,u}^{\mathrm{base}}
  +
  \lambda d_{b,u}^{\mathrm{trav}},
  \qquad \lambda\in[0,1].
  \label{eq:blockwise-interpolation}
\end{equation}
Coordinates are reconstructed by cumulative sums of $d_{b,u}^{(\lambda)}$. This preserves image simultaneity and video partition additivity for every $\lambda$. Direct interpolation with a flattened token index is excluded because it would reintroduce raster-order pseudo-time inside an image.

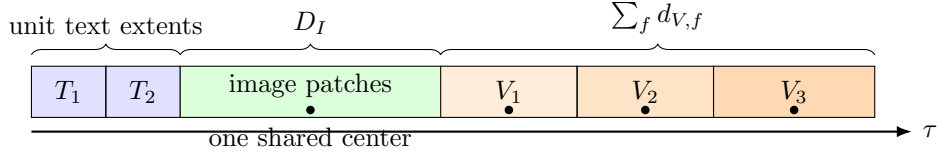
\begin{figure}[t]
\centering
\begin{tikzpicture}[x=0.82cm,y=0.75cm,font=\small]
  \draw[-{Latex[length=2mm]},thick] (0,0) -- (14.2,0) node[right] {$\tau$};
  \fill[blue!12] (0,0.25) rectangle (1.2,1.15);
  \fill[blue!12] (1.2,0.25) rectangle (2.4,1.15);
  \draw (0,0.25) rectangle (1.2,1.15);
  \draw (1.2,0.25) rectangle (2.4,1.15);
  \node at (0.6,0.7) {$T_1$};
  \node at (1.8,0.7) {$T_2$};
  \fill[green!14] (2.4,0.25) rectangle (6.6,1.15);
  \draw (2.4,0.25) rectangle (6.6,1.15);
  \node at (4.5,0.82) {image patches};
  \fill (4.5,0.38) circle (1.6pt);
  \node[below] at (4.5,0.25) {one shared center};
  \fill[orange!15] (6.6,0.25) rectangle (8.8,1.15);
  \fill[orange!22] (8.8,0.25) rectangle (11.0,1.15);
  \fill[orange!30] (11.0,0.25) rectangle (13.6,1.15);
  \draw (6.6,0.25) rectangle (8.8,1.15);
  \draw (8.8,0.25) rectangle (11.0,1.15);
  \draw (11.0,0.25) rectangle (13.6,1.15);
  \node at (7.7,0.7) {$V_1$};
  \node at (9.9,0.7) {$V_2$};
  \node at (12.3,0.7) {$V_3$};
  \fill (7.7,0.38) circle (1.6pt);
  \fill (9.9,0.38) circle (1.6pt);
  \fill (12.3,0.38) circle (1.6pt);
  \draw[decorate,decoration={brace,amplitude=4pt}] (0,1.35) -- node[above=4pt] {unit text extents} (2.4,1.35);
  \draw[decorate,decoration={brace,amplitude=4pt}] (2.4,1.35) -- node[above=4pt] {$D_I$} (6.6,1.35);
  \draw[decorate,decoration={brace,amplitude=4pt}] (6.6,1.35) -- node[above=4pt] {$\sum_f d_{V,f}$} (13.6,1.35);
\end{tikzpicture}
\caption{Traversal-axis construction. Text advances by unit blocks; every patch of one image shares its block center; video advances additively over temporal-token slices while H/W patches within each slice remain simultaneous.}
\label{fig:traversal-axis}
\end{figure}

\subsection{Formal Properties}

\begin{proposition}[Pure-text preservation]
For a text-only sequence, traversal coordinates recover ordinary one-dimensional relative offsets exactly.
\end{proposition}
\begin{proof}
Every block is one text token with $D_b=1$. Under zero-based block indexing, $\tau_i=i+1/2$, hence $\tau_j-\tau_i=j-i$.
\end{proof}

\begin{proposition}[Image simultaneity]
For any two patches $i,j$ in the same image, $\tau_j-\tau_i=0$ independently of flattening order.
\end{proposition}
\begin{proof}
Both tokens are assigned the common block center by Eq.~\eqref{eq:text-image-tau-new}.
\end{proof}

\begin{proposition}[Temporal-partition additivity]
Partition a video after temporal token $F_1$, without changing its post-tokenization slices. If the two segments contain the first $F_1$ and remaining $F_2$ slices, respectively, then
\begin{equation}
  D^{V}(1{:}F_1+F_2)
  =
  D^{V}(1{:}F_1)+D^{V}(F_1+1{:}F_1+F_2).
  \label{eq:video-additivity}
\end{equation}
\end{proposition}
\begin{proof}
Both sides are the same sum of the per-slice terms in Eq.~\eqref{eq:video-duration-general}. The property also holds for Eq.~\eqref{eq:blockwise-interpolation} because interpolation is applied before summation.
\end{proof}

\begin{proposition}[Local video-scale consistency]
For a fixed H/W grid, the offset between temporal slices $f$ and $f+k$ is
\begin{equation}
  \tau_{f+k}-\tau_f
  =
  kd_V^\star
  \left(
    \frac{\ell(H,W)}{\ell(H_V^\star,W_V^\star)}
  \right)^{\beta_V},
  \label{eq:local-scale-consistency}
\end{equation}
which is independent of the total video length $F_b$.
\end{proposition}
\begin{proof}
Subtract Eq.~\eqref{eq:video-tau-fixed} at the two slice indices. The anchor and half-slice offsets cancel.
\end{proof}

\begin{proposition}[Structural monotonicity]
$D_b^I$ increases under uniform enlargement of the image grid. $D_b^V$ increases strictly with $F_b$; it is independent of H/W when $\beta_V=0$ and increases under uniform spatial enlargement when $\beta_V>0$.
\end{proposition}
\begin{proof}
The result follows from positivity of the reference extents, scale covariance in Eq.~\eqref{eq:scale-covariance}, and the power-law terms.
\end{proof}

\subsection{Single-Slice Video Is Not an Image}

The formulation does not impose the cross-modal boundary condition
\begin{equation}
  D^V(F=1,H,W)=D^I(H,W).
  \label{eq:no-degeneration}
\end{equation}
A one-temporal-token video remains in the video observation regime: its modality label, tokenizer path, calibration constant, and exponent are those of video. Consequently,
\begin{equation}
  D^V(1,H,W)
  =
  d_V^\star
  \left(
    \frac{\ell(H,W)}{\ell(H_V^\star,W_V^\star)}
  \right)^{\beta_V}
\end{equation}
need not equal Eq.~\eqref{eq:image-duration-new}. This is deliberate. A static image is modeled as one inspectable spatial object, whereas a video temporal token is one slice of a stream representation, even when the stream happens to contain only one slice. Both are centered inside their own block, but equality of centers does not require equality of block extents. Enforcing equality for every H/W would force $\beta_V=1$, a shared scale function, identical reference grids, and matched calibration, eliminating the intended freedom to model video as a distinct observation regime.

\subsection{Representation Time and Physical Time}

Equations~\eqref{eq:F-definition}--\eqref{eq:video-duration-fixed} define \emph{representation time}. If the same physical clip is tokenized into more temporal tokens, its model-space extent increases. This behavior is intentional in the core method: each retained temporal token is an additional ordered representation slice. It should not be described as a claim that denser sampling makes the physical clip longer or literally slows playback.

Applications requiring sampling-rate invariance or timestamp fidelity can replace the unit ordered measure by timestamp weights. If temporal token $f$ covers physical interval $\Delta t_f$ and $t_0$ is a reference duration, an optional physical-time-aware extension is
\begin{equation}
  D_{b,\mathrm{phys}}^{V}
  =
  d_V^\star
  \sum_{f=1}^{F_b}
  \frac{\Delta t_f}{t_0}
  \left(
    \frac{\ell(H_{b,f},W_{b,f})}{\ell(H_V^\star,W_V^\star)}
  \right)^{\beta_V}.
  \label{eq:physical-time-extension}
\end{equation}
This extension changes the semantics of the axis and should be evaluated separately. The main RIG-RoPE formulation uses representation time.

\subsection{Relation to M-RoPE and RoPE Semantics}

M-RoPE already gives images a non-unit influence on subsequent positions: the next modality begins after the maximum temporal, height, or width ID of the preceding modality~\citep{wang2024qwen2vl}. For a roughly square image, its induced span is therefore comparable in order to $\max(H,W)\approx\sqrt{HW}$. Equation~\eqref{eq:image-duration-new} should not be presented as the first resolution-dependent image span. Its role is to make the span an explicit scalar metric, calibrate it to the base phase range through Eq.~\eqref{eq:phase-range-calibration}, and decouple it from relation-specific H/W coordinates. This calibration follows the broader requirement to preserve textual priors and positional coherence when modifying multimodal RoPE~\citep{huang2025revisiting}. The larger distinction appears for video: Eq.~\eqref{eq:video-duration-general} separates ordered temporal accumulation from parallel spatial scale and guarantees Eqs.~\eqref{eq:video-additivity} and~\eqref{eq:local-scale-consistency}.

Finally, a larger $|\tau_j-\tau_i|$ does not imply monotonically weaker attention. RoPE produces oscillatory relative phases:
\begin{equation}
  s_{ij}^{t}
  =
  \inner{\q_i^t}{R_t(\tau_j-\tau_i)\kvec_j^t}.
\end{equation}
Traversal extent therefore defines a \emph{relative-phase structural prior}, not an attention-decay law. Any claim of improved downstream attention or accuracy remains empirical.

\section{Method: RIG-RoPE}

\subsection{Relation Partition}

Figure~\ref{fig:relation-routing} summarizes the method. Pair typing determines which positional operator and within-group normalizer is used; a separate common statistic determines total group mass.

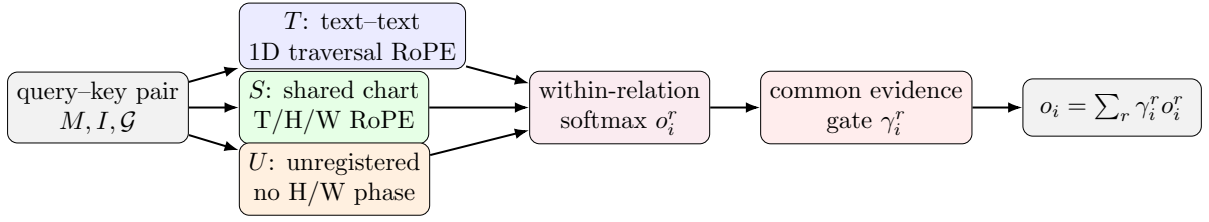
\begin{figure}[t]
\centering
\resizebox{\linewidth}{!}{%
\begin{tikzpicture}[
  node distance=5mm and 7mm,
  box/.style={draw,rounded corners,align=center,minimum height=8mm,minimum width=25mm,font=\small},
  arr/.style={-{Latex[length=2mm]},thick}
]
\node[box,fill=gray!10] (pair) {query--key pair\\$M,I,\mathcal G$};
\node[box,fill=blue!8,right=of pair,yshift=10mm] (t) {$T$: text--text\\1D traversal RoPE};
\node[box,fill=green!10,right=of pair] (s) {$S$: shared chart\\T/H/W RoPE};
\node[box,fill=orange!12,right=of pair,yshift=-10mm] (u) {$U$: unregistered\\no H/W phase};
\node[box,fill=purple!8,right=14mm of s] (norm) {within-relation\\softmax $o_i^r$};
\node[box,fill=red!7,right=of norm] (gate) {common evidence\\gate $\gamma_i^r$};
\node[box,fill=gray!10,right=of gate] (out) {$o_i=\sum_r\gamma_i^r o_i^r$};
\draw[arr] (pair) -- (t);
\draw[arr] (pair) -- (s);
\draw[arr] (pair) -- (u);
\draw[arr] (t) -- (norm);
\draw[arr] (s) -- (norm);
\draw[arr] (u) -- (norm);
\draw[arr] (norm) -- (gate);
\draw[arr] (gate) -- (out);
\end{tikzpicture}
}
\caption{Relation-stratified attention. Spatial phases are used only under a declared shared chart. Relation-specific scores rank keys within a group; H/W-neutral common evidence allocates total mass across active groups.}
\label{fig:relation-routing}
\end{figure}

Let $\mathcal{C}_i$ be the keys visible to query $i$ after the ordinary causal and padding masks. RIG-RoPE assigns every $j\in\mathcal{C}_i$ to exactly one relation class:
\begin{align}
  \calJ_i^T
  &=
  \{j\in\mathcal{C}_i:M_i=M_j=\Text\},\\
  \calJ_i^S
  &=
  \{j\in\mathcal{C}_i:\Shared_{\mathcal G}(i,j)\},\\
  \calJ_i^U
  &=
  \mathcal{C}_i\setminus(\calJ_i^T\cup\calJ_i^S).
  \label{eq:relation-partition}
\end{align}
Here $T$ denotes native text order, $S$ a declared spatial relation, and $U$ an unregistered relation. Under the default instance-local contract, $\Shared_{\mathcal G}(i,j)$ reduces to $M_i=M_j=\Vision$ and $I_i=I_j$. For a text query, only $T$ and $U$ can be nonempty. For a visual query, only $S$ and $U$ can be nonempty. Thus at most two groups are active for any query. Registered cross-instance pairs or intentionally shared normalized coordinates are supported by changing $\mathcal G$, not by changing the attention equations.

\subsection{Relation-Specific Logits}

Let
\begin{equation}
  \Delta\tau_{ij}=\tau_j-\tau_i.
  \label{eq:delta-tau}
\end{equation}
All logits below use the same $1/\sqrt d$ scaling.

For native text relations,
\begin{equation}
  s_{ij}^{T}
  =
  \frac{1}{\sqrt d}
  \inner{\q_i}{R_{\mathrm{1D}}(\Delta\tau_{ij})\kvec_j},
  \qquad j\in\calJ_i^T.
  \label{eq:text-score-v2}
\end{equation}
This uses the full head dimension and reduces exactly to ordinary text RoPE on text-only sequences.

For spatially registered visual relations,
\begin{align}
  s_{ij}^{S}
  =\frac{1}{\sqrt d}
  \big[
  &\inner{\q_i^t}{R_t(\Delta\tau_{ij})\kvec_j^t}
  +\inner{\q_i^h}{R_h(\Delta h_{ij}^{\mathcal G})\kvec_j^h}
  \nonumber\\[-2pt]
  &+\inner{\q_i^w}{R_w(\Delta w_{ij}^{\mathcal G})\kvec_j^w}
  \big],
  \qquad j\in\calJ_i^S.
  \label{eq:spatial-score-v2}
\end{align}
This retains M-RoPE H/W rotations on the domain declared spatially comparable while replacing the temporal-axis offset by $\Delta\tau_{ij}$. Because $\q^a=P_a\q$ is defined by frequency-block selectors, Eqs.~\eqref{eq:spatial-score-v2}--\eqref{eq:unregistered-score-v2} do not assume contiguous channel chunks and apply equally to interleaved or head-wise axis allocation.

For unregistered relations,
\begin{align}
  s_{ij}^{U}
  =\frac{1}{\sqrt d}
  \big[
  &\inner{\q_i^t}{R_t(\Delta\tau_{ij})\kvec_j^t}
  +\inner{\q_i^h}{\kvec_j^h}
  +\inner{\q_i^w}{\kvec_j^w}
  \big],
  \qquad j\in\calJ_i^U.
  \label{eq:unregistered-score-v2}
\end{align}
The identity on the H/W-selected blocks is an explicitly typed \emph{absence of H/W transformation}, not a claim that $\Delta h=\Delta w=0$. This distinction is operational rather than merely verbal: $s^U$ is never placed in the same within-group softmax as $s^S$. Traversal-aware order remains in the temporal blocks, while all content channels remain available. A full-dimensional temporal rotation or a learned content/position decomposition can be studied as an ablation, but is not required by the core formulation.

\subsection{Within-Relation Normalization}

For every nonempty relation group $r\in\{T,S,U\}$, define
\begin{equation}
  a_{ij}^{r}
  =
  \frac{\exp(s_{ij}^{r})}
       {\sum_{k\in\calJ_i^r}\exp(s_{ik}^{r})},
  \qquad
  o_i^r
  =
  \sum_{j\in\calJ_i^r}a_{ij}^{r}v_j.
  \label{eq:within-relation}
\end{equation}
The spatial phase ranks keys only against keys governed by the same spatial semantics. In particular, identity fallback can no longer steal probability mass from a valid nonzero H/W rotation.

\subsection{Common-Evidence Relation Gate}

Separating softmax denominators moves the remaining decision to the relation level. The gate must not reconstruct the same heterogeneous H/W competition under another name. Let $G(\Delta\tau)$ apply traversal rotation to blocks in $\Omega_t$ and identity to blocks in $\Omega_h\cup\Omega_w$. We define the common score
\begin{equation}
  c_{ij}
  =
  \frac{1}{\sqrt d}
  \inner{\q_i}{G(\Delta\tau_{ij})\kvec_j}.
  \label{eq:common-content}
\end{equation}
This statistic is relation independent, contains no H/W coordinate difference, remains sensitive to traversal order, and is well defined for contiguous, interleaved, and head-wise frequency assignments. Using only $P_t\q$ or an unrotated full-head content dot product are natural ablations.

For a nonempty group, the default evidence is the common-score LogSumExp
\begin{equation}
  e_i^r
  =
  \log
  \sum_{j\in\calJ_i^r}
  \exp(c_{ij}).
  \label{eq:group-evidence}
\end{equation}
With fixed positive relation priors $\pi_r$,
\begin{equation}
  \gamma_i^r
  =
  \frac{\pi_r\exp(e_i^r)}
       {\sum_{u\in\mathcal{A}_i}\pi_u\exp(e_i^u)},
  \qquad
  \mathcal{A}_i=\{r:|\calJ_i^r|>0\}.
  \label{eq:relation-gate}
\end{equation}
The fixed default uses uniform $\pi_r=1$ and introduces no learned gate parameters. The final head output is
\begin{equation}
  o_i
  =
  \sum_{r\in\mathcal{A}_i}\gamma_i^r o_i^r.
  \label{eq:hierarchical-output}
\end{equation}
This default preserves the ordinary role of group cardinality: if two groups contain equally scored keys, the larger group receives proportionally more total mass. When explicit group balancing is desired, we instead define the optional variant
\begin{equation}
  e_{i,\mathrm{bal}}^r(\tau_g)
  =
  \tau_g\log\left[
    \frac{1}{|\calJ_i^r|}
    \sum_{j\in\calJ_i^r}
    \exp\!\left(\frac{c_{ij}}{\tau_g}\right)
  \right],
  \qquad \tau_g>0.
  \label{eq:balanced-evidence}
\end{equation}
Unlike Eq.~\eqref{eq:group-evidence}, Eq.~\eqref{eq:balanced-evidence} encodes the deliberate prior that groups with equal average evidence deserve equal total mass regardless of key count. The outer factor $\tau_g$ preserves the evidence scale as temperature changes. Optional per-head priors or query-conditioned gates add learned parameters and are extensions.

\subsection{Properties}

\begin{proposition}[Full connectivity and normalization]
Every visible key $j\in\mathcal{C}_i$ receives the final weight
\begin{equation}
  \alpha_{ij}=\gamma_i^{r_{ij}}a_{ij}^{r_{ij}},
\end{equation}
with $\alpha_{ij}>0$ for finite logits and
\begin{equation}
  \sum_{j\in\mathcal{C}_i}\alpha_{ij}=1.
\end{equation}
\end{proposition}

\begin{proof}
The sets in Eq.~\eqref{eq:relation-partition} form a disjoint partition. Each within-group softmax sums to one, and the relation gate sums to one over active groups. Therefore
\begin{equation}
  \sum_{j\in\mathcal{C}_i}\alpha_{ij}
  =
  \sum_{r\in\mathcal{A}_i}\gamma_i^r
  \sum_{j\in\calJ_i^r}a_{ij}^r
  =
  \sum_{r\in\mathcal{A}_i}\gamma_i^r=1.
\end{equation}
\end{proof}

\begin{proposition}[Common-score consistency]
Assume uniform relation priors and suppose every relation-specific score reduces to the common score, $s_{ij}^{r_{ij}}=c_{ij}$. Then the final weight equals standard global softmax on $c$:
\begin{equation}
  \alpha_{ij}
  =
  \frac{\exp(c_{ij})}
       {\sum_{k\in\mathcal C_i}\exp(c_{ik})}.
  \label{eq:common-score-consistency}
\end{equation}
\end{proposition}

\begin{proof}
Let $Z_i^r=\sum_{j\in\calJ_i^r}\exp(c_{ij})$. Equation~\eqref{eq:group-evidence} gives $\exp(e_i^r)=Z_i^r$, so $\gamma_i^r=Z_i^r/\sum_u Z_i^u$. Multiplying by the within-group weight $\exp(c_{ij})/Z_i^r$ cancels $Z_i^r$ and yields Eq.~\eqref{eq:common-score-consistency}.
\end{proof}

\begin{proposition}[Explicit relation-mass control]
For every active relation $r$,
\begin{equation}
  \sum_{j\in\calJ_i^r}\alpha_{ij}=\gamma_i^r.
  \label{eq:mass-control}
\end{equation}
Consequently, arbitrary changes to the positional logits $s_{ij}^r$ can redistribute mass inside group $r$ but cannot change its total mass unless the common gate evidence changes.
\end{proposition}

\begin{proof}
Sum $\alpha_{ij}=\gamma_i^r a_{ij}^r$ over $j\in\calJ_i^r$ and use $\sum_j a_{ij}^r=1$.
\end{proof}

\begin{corollary}[Removal of cross-group null-phase privilege]
Consider one spatially registered key and one unregistered key with identical common evidence $c_{ij}$, equal priors, and no other visible keys. Then $\gamma_i^S=\gamma_i^U=1/2$, independent of the valid H/W displacement $\Delta$. The unregistered key cannot gain cross-group mass merely because its H/W branch uses identity.
\end{corollary}

\begin{proposition}[Gauge invariance of unregistered interactions]
Without a registration map, the unregistered logit in Eq.~\eqref{eq:unregistered-score-v2}, its common gate evidence in Eq.~\eqref{eq:common-content}, and its final relation mass are invariant to independent translations of the visual coordinate charts.
\end{proposition}

\begin{proof}
Neither $s_{ij}^U$ nor $c_{ij}$ contains $h_j^{(n)}-h_i^{(m)}$ or $w_j^{(n)}-w_i^{(m)}$. Therefore independent chart-origin shifts do not change the within-group weights, evidence gate, or final output through the unregistered branch.
\end{proof}

\subsection{Relation-Specific Versus Common-Score LogSumExp}

Any flat softmax can be factorized into within-group softmaxes. Let
\begin{equation}
  Z_i^r=\sum_{j\in\calJ_i^r}\exp(s_{ij}^r).
\end{equation}
Choosing
\begin{equation}
  \gamma_{i,\mathrm{flat}}^r
  =
  \frac{Z_i^r}{\sum_u Z_i^u}
  \label{eq:lse-reconstruction}
\end{equation}
reconstructs exactly the original global softmax over the heterogeneous branch scores. This identity is useful for split-kernel implementations, but it also reconstructs the original calibration problem: relation-specific H/W phases again determine how much total mass every other group receives. RIG-RoPE also uses LogSumExp, but applies it to the \emph{same common score} $c_{ij}$ in every group. The distinction is the score being pooled, not the pooling operator. Equation~\eqref{eq:common-score-consistency} then guarantees recovery of ordinary global softmax whenever branch scores coincide with that common score.

\section{Implementation}

\subsection{Traversal-Coordinate Construction}

Traversal coordinates are preprocessing metadata computed after visual tokenization and before the transformer stack. For each packed example, the packer performs the following operations:
\begin{enumerate}[leftmargin=1.5em]
  \item segment the interleaved input into text-token, image-instance, and video-instance blocks;
  \item read the post-tokenization grid of each visual block, using $F_b$ from the temporal output dimension of the visual tokenizer;
  \item evaluate Eqs.~\eqref{eq:text-duration}, \eqref{eq:image-duration-new}, and~\eqref{eq:video-local-duration};
  \item compute block anchors by a prefix sum of slice extents;
  \item assign one center coordinate to each text token or image and one center coordinate to each video temporal token using Eqs.~\eqref{eq:text-image-tau-new} and~\eqref{eq:video-tau-new}; and
  \item broadcast the coordinate of each visual slice to all of its H/W tokens.
\end{enumerate}
Model-control tokens that do not belong to a visual grid are treated as unit text-like blocks unless a model-specific packing rule assigns them to an adjacent semantic block. Fractional $\tau_i$ values require no change to RoPE: sine and cosine phases are evaluated at real-valued coordinates. The construction is a prefix scan over block metadata and is $O(L)$ in the packed sequence length.

\subsection{Tiled Relation Routing}

A naive implementation could materialize three masks of shape $L\times L$. This is unnecessary. RIG-RoPE requires only
\begin{equation}
  M\in\{0,1\}^{L},
  \qquad
  I\in\mathbb{N}_0^{L},
  \qquad
  \tau\in\R^{L},
\end{equation}
plus the ordinary H/W coordinates already required by M-RoPE. The traversal coordinate is computed once during multimodal packing. In a FlashAttention-style tiled kernel~\citep{dao2022flashattention,dao2023flashattention2}, the default relation class is computed in registers from $(M_i,M_j,I_i,I_j)$.

The tuple $(M,I)$ implements the default instance-local contract. If $\mathcal G$ contains cross-instance registration, the packer additionally supplies a compact chart ID and coordinates already mapped into that chart; pair classification then compares chart IDs instead of only instance IDs. Dense pairwise registration metadata is outside the fixed-overhead default.

For each query row, at most two relation groups are active. A fused kernel can maintain a separate online-softmax state $(m_i^r,\ell_i^r,o_i^r)$ for each active group, together with a common-score LogSumExp state for $c_{ij}$. Each query--key pair is classified once and updates exactly one group state. After all key tiles have been processed, the kernel finalizes the within-group outputs, evaluates Eq.~\eqref{eq:relation-gate}, and mixes the two outputs. The optional balanced variant additionally tracks the count, subtracts $\log|\calJ_i^r|$, and applies the temperature scaling in Eq.~\eqref{eq:balanced-evidence}. Thus relation stratification does not require two complete passes over the attention matrix, although a simpler two-kernel implementation is possible.

\subsection{Pseudocode}

\begin{verbatim}
initialize online softmax and evidence states for active groups

for Q_block, K_block in tiled_attention:
    load q, k, v and metadata M, I, tau, h, w
    delta_tau = tau_k - tau_q
    common = dot(q, G(delta_tau) * k) / sqrt(d)

    relation = classify(M_q, M_k, shared_chart_qk)  # T, S, or U

    score_T = full_1d_rope_score(q, k, delta_tau)
    score_S = rope_t_score(q_t, k_t, delta_tau) \
              + rope_h_score(q_h, k_h, delta_h_G) \
              + rope_w_score(q_w, k_w, delta_w_G)
    score_U = rope_t_score(q_t, k_t, delta_tau) \
              + dot(q_h, k_h) + dot(q_w, k_w)

    update the online-softmax state selected by relation
    update that group's common-score logsumexp

for each active relation r:
    output_r = finalize_online_softmax(r)
    evidence_r = evidence_lse_r

gamma = softmax(evidence + log(relation_prior))
output = sum_r gamma_r * output_r
\end{verbatim}

\subsection{Complexity}

The asymptotic compute remains $O(L^2d)$ and the persistent metadata overhead is $O(L)$ under the default instance-local contract. The fused implementation stores a constant number of extra online-softmax scalars and output accumulators per query row; it never materializes dense relation masks. The common evidence shares its traversal-rotated temporal term with every branch, but its unrotated H/W content term may require extra arithmetic for registered pairs. Compared with a conventional kernel, the principal costs are that common-score evaluation, relation classification, two normalization states instead of one, and less uniform control flow. A split-kernel implementation is simpler but can increase memory traffic and launch overhead. These constants require empirical GPU benchmarking and are not claimed to be negligible.

\section{Controlled Validation Results and Remaining Protocol}

This report separates theoretical formulation, completed controlled evidence, and large-scale empirical claims. The completed results below test a specific H/W Gauge mechanism, including a matched native/RIG run in a real Qwen2-VL checkpoint. They do not validate task utility for the common gate or traversal coordinates and are not a task benchmark. The remaining subsections retain the broader protocol for work not performed in this report.

\subsection{Completed Controlled H/W Gauge Evidence}

\paragraph{Full-checkpoint baseline activation.}
We ran the underlying Qwen2-VL model in BF16, eager-attention, evaluation, inference-only mode, without a RIG patch. In the pure-coordinate pair, pixels, processor payload, token layout, and attention mask were identical, and the visual embeddings were exactly identical (maximum absolute difference $0$). Only the explicit H/W rows of the position IDs changed. The embedding/pre-decoder hidden state remained identical, but all 28 subsequent decoder hidden states and all 28 attention-layer summaries changed. This establishes that a content-preserving raw-coordinate change can propagate through this baseline checkpoint. It is not a RIG-versus-baseline task comparison, a causal task result, or evidence of superiority.

\paragraph{Matched full-checkpoint RIG mechanism test.}
We next compared native and RIG attention with the same Qwen2-VL-2B-Instruct checkpoint, inputs, BF16 dtype, eager backend, evaluation mode, and inference-only execution. Table~\ref{tab:qwen-rig-gauge-results} states the purpose and intervention for each frozen comparison. Maximum differences are taken over all recorded attention values, hidden states, and final language-model logits. Text-only native and RIG paths were exactly equal. For RIG, translating every visual H/W coordinate of one image by $(+7,-5)$, or translating only the second instance in a pair of unregistered images, left all three recorded families exactly unchanged. Applying the latter intervention to native attention produced a nonzero response. Collapsing every visual H/W coordinate in one image to the first visual token's coordinate also changed the RIG activations and logits, confirming that the same-instance spatial branch was not made H/W-insensitive.

\begin{table}[t]
\centering
\scriptsize
\setlength{\tabcolsep}{2.5pt}
\caption{Matched Qwen2-VL-2B-Instruct BF16/eager, inference-only mechanism results. Entries are maximum absolute differences (nonzero controls rounded to six decimals); zeros are exact. The first three rows are equivalence/invariance checks, and the final two are sensitivity controls.}
\label{tab:qwen-rig-gauge-results}
\begin{tabular}{>{\raggedright\arraybackslash}p{0.25\linewidth}>{\raggedright\arraybackslash}p{0.31\linewidth}>{\centering\arraybackslash}p{0.10\linewidth}>{\centering\arraybackslash}p{0.08\linewidth}>{\centering\arraybackslash}p{0.10\linewidth}}
\toprule
Comparison purpose & Frozen intervention & Attention & Hidden & Logits \\
\midrule
Text-path degeneration & Native versus RIG on text only & $0$ & $0$ & $0$ \\
Same-instance Gauge invariance & RIG; translate one image's whole chart by $(H{+}7,W{-}5)$ & $0$ & $0$ & $0$ \\
Unregistered-instance Gauge invariance & RIG; translate only image 2 in an unregistered two-image input & $0$ & $0$ & $0$ \\
Native sensitivity control & Native; apply the same image-2 translation & $0.256104$ & $2.0$ & $0.5$ \\
RIG spatial-retention control & RIG; collapse one image's H/W coordinates to its first patch & $0.103516$ & $2.0$ & $0.321289$ \\
\bottomrule
\end{tabular}
\end{table}

The native and RIG paths used the same model object. Single-image and two-image visual embeddings were exactly identical between them, and all 729 parameters preserved object, storage, bytes, and gradient-flag identity. The run took $83.9771748459898$ seconds on one idle H20, with peak allocated/reserved CUDA memory of $4{,}668{,}478{,}976/5{,}016{,}387{,}584$ bytes. No training, checkpoint write, raw tensor, task answer, accuracy, or negative log-likelihood was produced or persisted. This is activation/mechanism evidence for the frozen operator, not task utility, training benefit, Gauge-AUC, universality, or empirical superiority.

\paragraph{Frozen tiny-model comparison.}
We then trained a raw-H/W global-softmax baseline and the fixed dense RIG construction from scratch under matched data order, optimizer, width, and 200-epoch full-batch budget. The synthetic target is a cyclic relation between two visual instances; a third value is a label-independent decoy. The fixed Gauge intervention changes irrelevant chart origins without changing the relational content. Table~\ref{tab:tiny-gauge-results} reports validation metrics. Seed 0 is an exact deterministic reproduction of the already-inspected single-seed Pilot, not an independent new evidence unit.

\begin{table}[t]
\centering
\small
\setlength{\tabcolsep}{4pt}
\caption{Frozen tiny-model validation results. Accuracy cells are clean/Gauge. ``Logit diff'' is the maximum absolute clean-to-Gauge logit change. The paired delta is RIG minus raw-H/W Gauge accuracy.}
\label{tab:tiny-gauge-results}
\begin{tabular}{c c c c c c}
\toprule
Seed & Raw-H/W acc. & RIG acc. & Raw-H/W logit diff & RIG logit diff & Paired Gauge $\Delta$ \\
\midrule
0 & $0.9583/0.9444$ & $0.9722/0.9722$ & $1.211458$ & $0$ & $+0.027778$ ($+2/72$) \\
1 & $1.0000/1.0000$ & $1.0000/1.0000$ & $3.862063$ & $0$ & $0$ \\
2 & $0.9861/0.9861$ & $0.9861/0.9861$ & $3.677815$ & $0$ & $0$ \\
\bottomrule
\end{tabular}
\end{table}

The RIG logits were exactly invariant for all three seeds, whereas the raw-H/W baseline logits changed for all three. The preregistered task-stability gate nevertheless required a strictly positive Gauge-accuracy delta in at least two of three seeds and failed because only seed 0 was positive. Train and validation each contain all 36 ordered A/B combinations and differ only in their fixed decoy values. Validation therefore measures a decoy split, not unseen-relation, unseen-combination, or broad compositional generalization. The test manifest was structurally checked, but no trained model consumed test rows and no test model metric was computed or read.

\subsection{Remaining Null-Privilege and Mass-Control Protocol}

Construct synthetic query and key vectors with controlled content similarity. For each RoPE frequency, compare a valid pair $s^S(\Delta)$ with identity, Gaussian, zero, and relation-stratified treatments of an unregistered pair. Vary:
\begin{enumerate}[leftmargin=1.5em]
  \item the valid spatial displacement $\Delta$,
  \item cosine similarity between $\q_i$ and $\kvec_j$,
  \item the number of keys in each relation group,
  \item duplication of all keys in one group under both default and group-balanced evidence,
  \item fixed and learned relation priors.
\end{enumerate}
The analytic checks should recover Eq.~\eqref{eq:null-privilege-gap}, test the assumption dependence of Eq.~\eqref{eq:expected-null-privilege}, verify Eq.~\eqref{eq:mass-control}, and verify exact recovery of global common-score softmax in Eq.~\eqref{eq:common-score-consistency}. Exact duplication should increase default evidence by $\log 2$ while leaving the balanced evidence unchanged, exposing the intended difference between global-mass and group-balanced semantics. A direct comparison against Eq.~\eqref{eq:lse-reconstruction} should show that relation-specific LSE reconstruction preserves heterogeneous H/W competition whereas common-score LSE does not use H/W displacement to allocate cross-group mass.

\subsection{Further Gauge Perturbation Protocol}

Apply independent coordinate-origin shifts to two images while preserving their token content. Static cross-instance M-RoPE changes its H/W phase by the relative gauge shift. RIG-RoPE should leave $s^U$, $a^U$, and $\gamma^U$ unchanged. Padding and resizing can be reported separately because they may also change visual features and same-instance geometry; they are not pure gauge tests.

\subsection{Remaining Traversal-Coordinate Consistency Protocol}

The coordinate implementation should first be tested without a trained model. Required checks are:
\begin{enumerate}[leftmargin=1.5em]
  \item verify exact text-only offsets $\tau_j-\tau_i=j-i$;
  \item verify zero traversal offset between every pair of patches in the same image, independently of raster order;
  \item vary image H/W resolution and aspect ratio, comparing $\ell_{\mathrm{area}}$, $\ell_{\mathrm{diag}}$, and $\ell_{\max}$ against the maximum-axis span induced by the reference M-RoPE implementation;
  \item split a tokenized video at arbitrary temporal-token boundaries and verify Eq.~\eqref{eq:video-additivity} to numerical precision;
  \item append temporal tokens while holding H/W fixed and verify that an existing $k$-slice offset is unchanged, as required by Eq.~\eqref{eq:local-scale-consistency}; and
  \item encode an image and a one-temporal-token video with the same H/W grid and verify that their extents follow their separate modality parameters rather than an imposed degeneration rule.
\end{enumerate}

The principal temporal ablation compares
\begin{equation}
  (FHW)^{1/3},
  \qquad
  F,
  \qquad
  F\left(\frac{\ell(H,W)}{\ell_V^\star}\right)^{\beta_V},
  \qquad
  \sum_f \frac{\Delta t_f}{t_0}
  \left(\frac{\ell(H_f,W_f)}{\ell_V^\star}\right)^{\beta_V},
\end{equation}
with $d_V^\star$ included in the implemented versions. The first tests isotropic token-volume scaling, the second sets $\beta_V=0$, the third is the representation-time proposal, and the fourth is the timestamp-aware extension. Sweep $\beta_V\in\{0,1/4,1/2,3/4,1\}$, anchor $d_V^\star$ by Eq.~\eqref{eq:phase-range-calibration}, and then evaluate other lengths, resolutions, and aspect ratios. This separates temporal additivity, the spatial correction, and the choice of characteristic scale.

\subsection{Remaining Task-Evaluation Protocol}

The completed full-checkpoint run verifies only the frozen activation/mechanism gates. A separately authorized study would still need to evaluate an M-RoPE-based model on tasks without full retraining:
\begin{enumerate}[leftmargin=1.5em]
  \item text-only reasoning to verify traversal coordinates reduce to ordinary 1D RoPE,
  \item single-image spatial VQA to verify same-instance spatial reasoning remains intact,
  \item interleaved multi-image QA to test whether cross-instance coordinate leakage and null-relation privilege are reduced,
  \item interleaved text-image-text and text-video-text probes to test traversal scaling,
  \item video pairs with equal total token counts but different $(F,H,W)$ factorizations, and
  \item the same physical clip under multiple temporal sampling densities to compare representation time with timestamp-aware behavior.
\end{enumerate}
Required spatial ablations are flat softmax with identity fallback, Gaussian fallback, zero fallback, split kernels with relation-specific LSE reconstruction, equal relation mixing, default common-score LSE, the group-balanced variant, and optional learned relation priors. Required traversal ablations are the four coordinate families above, the three scale functions in Eq.~\eqref{eq:spatial-scales}, slice-level interpolation from Eq.~\eqref{eq:blockwise-interpolation}, and fixed versus learned reference extents. Report task accuracy, per-group attention mass, coordinate ranges, and extrapolation beyond the calibration grid; otherwise an apparent gain cannot be attributed to the proposed mechanism. This stage should be treated as sanity validation rather than final evidence.

\section{Related Work}

\paragraph{Rotary position encoding.}
RoPE was introduced in RoFormer~\citep{su2021roformer} and has become a standard position encoding mechanism in LLMs because it naturally induces relative-position structure in attention.

\paragraph{Multimodal RoPE and frequency allocation.}
Qwen2-VL introduces M-RoPE, assigning temporal, height, and width components to text, images, and videos at dynamic resolution~\citep{wang2024qwen2vl}. Its next modality starts after the maximum ID of the preceding modality, so RIG-RoPE does not claim that M-RoPE gives every image unit extent. A systematic study of multimodal RoPE identifies positional coherence, full frequency utilization, and preservation of textual priors as key principles, and proposes MRoPE-Interleave and MHRoPE~\citep{huang2025revisiting}. Qwen3-VL adopts enhanced interleaved M-RoPE~\citep{bai2025qwen3vl}. RIG-RoPE's frequency-set notation is intended to compose with these allocation strategies rather than replace them.

\paragraph{Visual position scaling.}
V2PE assigns variable, smaller position increments to visual tokens for multimodal long context~\citep{ge2024v2pe}. MODIX adapts positional granularity at inference time from covariance-based entropy and cross-modal alignment~\citep{huang2026modix}. RIG traversal addresses a related scale question but uses a different prior: under a fixed tokenizer, extent depends only on modality regime and representation geometry, not on semantic information estimates. This content independence is a design choice, not a claim that adaptive information-aware scaling is invalid.

\paragraph{Cross-modal positional bias.}
Circle-RoPE identifies cross-modal positional bias and proposes a cone-like/circular geometry so text tokens maintain equal distance to image tokens~\citep{wang2025circlerope}. Therefore the observation that cross-modal H/W phases can be problematic is not itself unique to RIG-RoPE. RIG-RoPE instead emphasizes the shared-chart contract, default instance-local invariance, and calibration of relation-specific normalizers.

\paragraph{Relation-conditioned multimodal position encoding.}
DIPE is the closest architectural neighbor: it separates intra- and inter-modal attention, applies different position operators, and fuses split-kernel outputs with LogSumExp statistics~\citep{chen2026dipe}. RIG-RoPE differs in three specific respects. Its partition is chart/instance aware rather than modality only; its group mass is computed from a common H/W-neutral score rather than relation-specific logits; and it adds ordered/parallel traversal theory. These are narrower claims than treating all relation-conditioned attention as new.

\paragraph{Video position encoding.}
VRoPE and VideoRoPE study rotary designs for video LLMs, with emphasis on spatiotemporal coherence and video--text transition behavior~\citep{liu2025vrope,videorope2025}. HoPE studies hybrid frequency allocation and dynamic temporal scaling for long video understanding~\citep{li2025hope}. Qwen3-Omni uses interleaved TM-RoPE, while Qwen3.5-Omni reports that sparse absolute time IDs can weaken long-range modeling and instead inserts explicit textual timestamps~\citep{xu2025qwen3omni,qwen2026qwen35omni}. These works motivate the distinction between representation time and physical timestamps in Section~\ref{sec:duration}.

\begin{table}[t]
\centering
\small
\setlength{\tabcolsep}{4pt}
\caption{Conceptual boundary between representative multimodal RoPE methods. ``Common gate'' means group mass is computed from a relation-independent score.}
\label{tab:method-comparison}
\begin{tabular}{>{\raggedright\arraybackslash}p{0.15\linewidth}>{\raggedright\arraybackslash}p{0.21\linewidth}>{\raggedright\arraybackslash}p{0.20\linewidth}>{\raggedright\arraybackslash}p{0.18\linewidth}>{\raggedright\arraybackslash}p{0.17\linewidth}}
\toprule
Method & Pairwise geometry & Cross-instance policy & Cross-group mass & Global traversal \\
\midrule
M-RoPE & fixed T/H/W axes & preprocessing convention & flat softmax & axis-maximum advance \\
Circle-RoPE & cross-modal circular/cone geometry & not instance-routed & flat softmax & inherited sequence design \\
DIPE & intra-/inter-modal operators & modality-level split & relation-score LSE & anchored inter-modal distance \\
RIG-RoPE & declared shared chart & instance-local default; registration allowed & common-score LSE & ordered/parallel additive metric \\
\bottomrule
\end{tabular}
\end{table}

\paragraph{Efficient attention kernels.}
FlashAttention and FlashAttention-2 show that exact attention can be made memory-efficient by tiling computation around GPU memory hierarchy~\citep{dao2022flashattention,dao2023flashattention2}. RIG-RoPE is intended for this execution model: relation gates can be computed inside each tile without materializing dense masks.

\section{Limitations}

This work establishes a geometric argument, a concrete algorithmic rule, and controlled real-checkpoint mechanism evidence, but it does not include a task-level Qwen benchmark, Qwen training, or large-scale empirical results. Several limitations should be addressed before claiming practical superiority:
\begin{enumerate}[leftmargin=1.5em]
  \item \textbf{Mechanism stability does not imply accuracy stability.} In the frozen three-seed tiny task, RIG's exact Gauge logit invariance and the raw-H/W baseline's sensitivity replicated across seeds, but the Gauge-accuracy advantage was positive in only one seed and tied in two. The preregistered stability gate therefore failed. The theoretical invariants remain intact, but stable task improvement is unestablished.
  \item \textbf{Real-model boundary.} The matched full-checkpoint Qwen2-VL-2B result verifies exact text-only native/RIG equivalence, two RIG Gauge invariances, native Gauge sensitivity, and retained same-instance H/W sensitivity for one frozen inference configuration. It contains no training, task metric, public benchmark, or multiple-model replication. It therefore cannot establish downstream benefit, post-training behavior, or generality beyond this checkpoint and intervention family.
  \item \textbf{Gate sufficiency and cardinality prior.} The common evidence in Eq.~\eqref{eq:common-content} is deliberately H/W-coordinate-neutral, but its task utility has not been established. The default retains key-count contribution; Eq.~\eqref{eq:balanced-evidence} instead imposes group balance. Unrotated evidence, learned per-head priors, or query-conditioned gates may improve accuracy at the cost of parameters and weaker analytic isolation.
  \item \textbf{Changed attention semantics.} Relation-stratified normalization is equivalent to global attention only in the common-score limit of Eq.~\eqref{eq:common-score-consistency}. With relation-specific scores it intentionally changes attention semantics, so pretrained weights are not guaranteed to remain calibrated after a training-free patch.
  \item \textbf{Unregistered H/W content.} Equation~\eqref{eq:unregistered-score-v2} retains semantic content from the H/W chunks via identity. Separate normalization removes its cross-group phase privilege, but the best unregistered branch parameterization remains empirical. Zero, Gaussian, full-temporal, and explicitly decoupled content branches should be ablated.
  \item \textbf{Traversal calibration and utility.} The power-law form is a structured inductive bias rather than a uniquely derived law, and its task utility has not been tested. The choice of $\ell$, the video-family parameter $\beta_V$, the reference extents $D_I^\star,d_V^\star$, and the optional $\lambda$ alter phase ranges. Equation~\eqref{eq:phase-range-calibration} supplies a deterministic base-model anchor but does not determine the best extrapolation behavior.
  \item \textbf{Representation-time semantics.} The core video coordinate scales with the number of post-tokenization temporal tokens and is therefore not invariant to sampling density. It is neither physical wall-clock time nor measured cognitive time. Tasks requiring timestamp fidelity should evaluate Eq.~\eqref{eq:physical-time-extension} separately.
  \item \textbf{Cross-modal boundary.} A one-temporal-token video is intentionally not forced to share the image extent. This preserves modality-specific observation regimes but introduces additional calibration freedom that must be justified empirically.
  \item \textbf{Chart-contract dependence.} Instance-local geometry is a conservative default, not a universal truth. Tasks may benefit from canonical normalized screen coordinates even without physical registration. Such pairs should enter $S$ only when that convention is declared and tested; the gauge theorem does not rule out convention-dependent predictive features.
  \item \textbf{Video and registration boundaries.} Treating one video as a shared screen-space instance is approximate under camera motion, cropping, and object motion. Conversely, stereo pairs, aligned medical scans, and tracked frames may provide valid cross-instance registration and should be reclassified accordingly.
  \item \textbf{Kernel engineering.} A fused grouped-softmax kernel has the same asymptotic complexity but additional state and control flow. Throughput, memory traffic, numerical stability, causal masking, and KV-cache behavior require measurement on modern accelerators.
\end{enumerate}

\section{Conclusion}

RIG-RoPE is based on a simple principle: rotary phases and softmax denominators should reflect meaningful relation structure. An unregistered pair cannot be made intrinsic by assigning a special distance. Zero maximizes phase contribution in the self-aligned case; Gaussian uncertainty assumes a latent alignment that may not exist; and no fixed matrix reproduces every valid displacement pointwise. A typed relation followed by relation-homogeneous normalization is therefore one principled solution, though not the only possible calibration mechanism.

The revised method preserves full one-dimensional RoPE for text--text interactions and M-RoPE H/W rotations wherever $\mathcal G$ declares a shared visual chart, while replacing the temporal-axis coordinate by $\tau$. Unregistered text--vision and default cross-instance visual interactions remain fully connected but are normalized in their own branch. A common H/W-neutral LogSumExp gate allocates total mass across branches and exactly recovers global softmax in the common-score limit. Representation-aware traversal coordinates independently provide one scalar cross-block axis: text is preserved exactly, image patches are simultaneous, and video extent is additive over tokenizer-produced temporal slices with a bounded sublinear spatial correction. The construction is explicitly representation time rather than physical time, and a one-slice video need not degenerate to an image.

Controlled validation supports the narrower mechanism claim. In a matched full-checkpoint Qwen2-VL-2B inference run, native attention responded to a second-instance-only coordinate translation, whereas RIG was exactly invariant to that intervention and to a whole-chart single-image translation; text-only native/RIG equality, exact visual-embedding identity, exact parameter identity, and a nonzero H/W-collapse control constrain simpler implementation explanations. The frozen tiny RIG construction also removed the declared Gauge dependence from output logits across three seeds. Under that same frozen task protocol, paired Gauge-accuracy differences were $+2/72$, $0$, and $0$, so the preregistered stability gate failed and stable task utility remains unestablished. The formula-level conclusions and controlled mechanism evidence therefore stand, but downstream benefit, common-gate task utility, traversal task utility, training benefit, benchmark performance, universality, and empirical superiority remain unestablished.

\appendix
\section{Additional Null-Relation Results}
\label{app:null-results}

\subsection{No Universally Neutral Fixed Fallback}

One might replace $I$ by a fixed matrix $N$ chosen to be less favorable. No fixed operator can be pointwise neutral relative to every valid displacement.

\begin{theorem}[No pointwise universal fixed fallback]
For a nonzero RoPE frequency $\omega$, there is no fixed matrix $N\in\R^{2\times2}$ satisfying
\begin{equation}
  u^\top Nu
  =
  u^\top R(\omega\Delta)u
  \label{eq:universal-null}
\end{equation}
for every $u\in\R^2$ and every valid displacement $\Delta$.
\end{theorem}

\begin{proof}
Equation~\eqref{eq:universal-null} for every $u$ implies equality of the symmetric parts of the two matrices. The symmetric part of $R(\omega\Delta)$ is $\cos(\omega\Delta)I$. A fixed $N$ would therefore require
\begin{equation}
  \tfrac12(N+N^\top)=\cos(\omega\Delta)I
\end{equation}
for every $\Delta$, which is impossible for nonzero $\omega$ because the right-hand side varies with $\Delta$.
\end{proof}

Thus neutrality must be defined relative to a chosen distribution, task, or normalization structure; it cannot be obtained by discovering a special displacement such as $0$, $\infty$, or a single constant rotation. This theorem rules out pointwise equivalence to every valid rotation. It does not make relation-stratified normalization logically unique: learned null relations, relation-conditioned projections, or statistically calibrated biases remain possible alternatives.

\subsection{Why Gaussian Marginalization Is Insufficient}

Suppose a displacement is assumed to exist but is uncertain,
\begin{equation}
  \Delta\sim\mathcal{N}(0,\sigma^2).
\end{equation}
For one frequency block,
\begin{equation}
  \mathbb{E}[R(\omega\Delta)]
  =
  \exp\!\left(-\tfrac12\omega^2\sigma^2\right)I
  =\alpha_\omega I,
  \qquad 0<\alpha_\omega\leq1.
  \label{eq:gaussian-fallback}
\end{equation}
For identical content, the fallback score is $\alpha_\omega\lVert u\rVert^2$. This attenuates identity fallback but does not make it neutral relative to all $\cos(\omega\Delta)$. More importantly, the prior describes a latent relation that exists and is concentrated around alignment. It is appropriate for noisy registration, not for a pair for which the modeling contract declares no shared chart.

Taking $\sigma\rightarrow\infty$ yields $\alpha_\omega\rightarrow0$, which removes both spatial phase and semantic evidence carried in that channel and changes the logit variance. Moreover, logit marginalization does not equal attention marginalization:
\begin{equation}
  \mathbb{E}_{\Delta}[\softmax(s(\Delta))]
  \neq
  \softmax\!\left(\mathbb{E}_{\Delta}[s(\Delta)]\right)
  \label{eq:softmax-expectation}
\end{equation}
in general. Gaussian and zero operators remain useful ablations for partially registered and conservative settings, respectively, but neither is universally neutral for an unregistered relation.

\end{document}